\documentclass[11pt,reqno]{amsart}

\usepackage[T1]{fontenc}
\usepackage[utf8]{inputenc}
\usepackage{amsmath,amssymb,amsfonts,amsthm,bbm,mathrsfs}
\usepackage{array}
\usepackage{enumitem}
\usepackage{xcolor}
\usepackage{microtype}
\usepackage[hidelinks]{hyperref}
\allowdisplaybreaks
\numberwithin{equation}{section}

\newtheorem{tw}{Theorem}[section]
\newtheorem{corollary}[tw]{Corollary}
\newtheorem{proposition}[tw]{Proposition}
\newtheorem{lemma}[tw]{Lemma}
\theoremstyle{definition}
\newtheorem{definition}[tw]{Definition}
\newtheorem{example}[tw]{Example}
\theoremstyle{remark}
\newtheorem{remark}[tw]{Remark}

\newcommand\sA{\mathsf{A}}
\newcommand\bA{{\mathscr A}}
\newcommand\cE{{\mathcal E}}
\newcommand\bE{{\mathscr E}}
\newcommand\cF{{\mathcal F}}

\newcommand\bF{{\mathbf F}}

\newcommand{\bmu}{\boldsymbol{\mu}}
\newcommand\bpeA{{\mathscr A}_{\bullet|\bullet}}
\newcommand{\mI}[1]{\mathbbm{1}_{#1}}
\newcommand{\md}{\,\mathrm d}
\newcommand{\essinf}{\operatorname*{ess\,inf}}
\renewcommand\le{\leqslant}
\renewcommand\ge{\geqslant}
\newcommand{\Llambda}{\mathsf{\Lambda}}

\title[Context Localization for Generalized Level-Based Evaluation]{Context Localization for Generalized Level-Based Evaluation in Knowledge-Based Systems}
\thanks{Author preprint. This manuscript is not the publisher's version of record. The peer-reviewed version has been published in \emph{Knowledge-Based Systems}; see the \href{https://www.sciencedirect.com/science/article/pii/S0950705126015583}{published article on ScienceDirect}.}

\author{Ondrej Hutn\'{i}k}
\address{Institute of Mathematics, Faculty of Science, Pavol Jozef \v{S}af\'{a}rik University in Ko\v{s}ice, Jesenn\'{a} 5, 040 01 Ko\v{s}ice, Slovakia}
\email{ondrej.hutnik@upjs.sk}

\author{Nat\'{a}lia Pu\v{s}k\'{a}rov\'{a}}
\address{Institute of Mathematics, Faculty of Science, Pavol Jozef \v{S}af\'{a}rik University in Ko\v{s}ice, Jesenn\'{a} 5, 040 01 Ko\v{s}ice, Slovakia}
\email{natalia.puskarova@student.upjs.sk}

\hypersetup{
 pdftitle={Context Localization for Generalized Level-Based Evaluation in Knowledge-Based Systems},
 pdfauthor={Ondrej Hutnik and Natalia Puskarova}
}

\begin{document}
\begin{abstract}
We study context localization for generalized level-based evaluation in knowledge-based systems. The framework models situations where a structured nonnegative score, defined on facts, rules, cases, criteria or evidence units, is evaluated through conditional aggregation tests on admissible knowledge contexts. The generalized level measure maximizes a monotone set function over all contexts whose aggregated support reaches a prescribed level. We characterize when filtering the score by a context $B$ is equivalent to localizing the admissible contexts by intersection with $B$. The main theorem shows that this consistency holds for all monotone set functions if and only if two structural conditions are satisfied: monotonicity with respect to contexts and a reduction property excluding positive localized support outside $B$. We analyze pointwise and block-generated mechanisms producing the reduction property, extend the result to parameterized systems, and interpret it as a stability criterion for context-dependent evidence selection, non-additive support evaluation and level-based knowledge aggregation.
\end{abstract}

\maketitle

\section{Motivation: context localization in knowledge-based evaluation}

In many knowledge-based systems, an output score is not obtained from a single scalar value but from a structured collection of facts, rules, cases, criteria, expert statements, documents or evidence units. Such units are rarely evaluated in isolation. Their relevance depends on the active query, decision context, user profile, diagnostic regime, information source, semantic constraint or observation window. The resulting system-level evaluation should therefore depend not only on the local scores assigned to the units, but also on which subsets of the knowledge base are admissible as coherent contexts for evaluation. This viewpoint is compatible with fuzzy, non-additive and approximate-reasoning models, where uncertain or imprecise requirements are evaluated through nonlinear satisfaction and aggregation mechanisms; see, for instance, \cite{BellmanZadeh1970,Zimmermann1978}.

The localization problem considered in this paper can be stated informally as follows. If only a part $B$ of the knowledge base is relevant for the current task, one may either first filter the score function outside $B$ or, alternatively, evaluate admissible contexts in the original system and then localize their informational weight by intersection with $B$. These two procedures need not be equivalent for nonlinear conditional evaluation rules. Our goal is to characterize exactly when this equivalence holds.

The classical level-set construction already contains the basic localization principle in its simplest form. Let $f\colon X\to[0,\infty)$ be a nonnegative measurable score function and let $B\in\Sigma$ be the part of the domain that is active for the current task. For every $u>0$,
$$
\{x\in X:f(x)\mI{B}(x)\ge u\}
=
\{x\in X:f(x)\ge u\}\cap B.
$$
Consequently, for every monotone measure $\mu$ on the ambient measurable space,
$$
\mu\big(\{x\in X:f(x)\mI{B}(x)\ge u\}\big)
=
\mu\big(\{x\in X:f(x)\ge u\}\cap B\big),\qquad u>0.
$$
Thus, in the ordinary level-based setting, filtering the score outside $B$ before evaluation is equivalent to localizing the resulting level set by intersection with $B$.

The aim of the present paper is to identify when this elementary localization principle remains valid for generalized level measures generated by conditional aggregation operators. Let $\bA=\{\sA(\cdot|E):E\in\cE\}$ be a family of conditional aggregation operators in the sense of \cite{BHHK}, and let $\mu$ be a monotone measure on $\cE$. In a knowledge-based interpretation, the sets $E\in\cE$ are admissible knowledge contexts, $\sA(f|E)$ is the aggregated degree of support of the score function $f$ inside $E$, and $\mu(E)$ is a non-additive informational weight of the context. Following the generalized level-measure construction from \cite{BHKK}, for $u\ge0$, define
$$\Llambda_{\mu,\bA}(f,u)
:= \sup\{\mu(E): \sA(f|E)\ge u, E\in\cE\}.$$
This quantity may be interpreted as a level-constrained evidence selection value: at level $u$, one searches for admissible contexts $E$ on which the score passes the conditional aggregation test, and then takes the largest support, confidence or relevance mass $\mu(E)$.

The localized counterpart used in the paper is not obtained by simply replacing $\cE$ with the subcollection of sets contained in $B$. Instead, each admissible set $E$ is localized by intersection with $B$, i.e.,
$$\Llambda_{\mu,\bA}(f,u;B) :=
\sup\{\mu(E\cap B): \sA(f|E)\ge u, E\in\cE\}.$$
This is closer to the classical level-set identity, where the level set $\{f\ge u\}$ is replaced by its intersection with $B$. The central localization problem is therefore to determine when
$$\Llambda_{\mu,\bA}(f\mI{B},u)
= \Llambda_{\mu,\bA}(f,u;B), \quad u>0.$$
In words, we ask when filtering the score outside $B$ is equivalent to localizing the admissible knowledge contexts appearing in the generalized level measure.

This equality is automatic for the ordinary level measure, but it is not automatic for arbitrary conditional aggregation rules. Averaging-type tests, for instance, generally do not have this property, since zeros introduced outside $B$ may still affect an average computed on a set not contained in $B$. In contrast, infimum-type, essential-infimum-type, conjunctive, or suitable block-generated tests may possess a reduction mechanism: after multiplication by $\mI{B}$, a set not contained in $B$ cannot produce a positive level value. This is the role of the reduction property studied below.

A second structural requirement is also needed. Even if the masked signal eliminates sets outside $B$, one must still compare the value of the conditional aggregation test on a set $E$ with its value on the localized part $E\cap B$. This is controlled by monotonicity with respect to sets: enlarging the conditioning set should not increase the local evaluation value. The paper therefore first separates these two mechanisms. We discuss monotonicity with respect to sets in Section~\ref{subsec:antitonicity}, then the reduction property is investigated in Section~\ref{subsec:reduction-property}, and only afterwards combine them in the localization theorem Section~\ref{subsec:localization-theorem}.

A parameterized version is natural in knowledge-based systems because the active evaluation structure may depend on a query, decision context, user profile, resolution level or admissibility regime. Changing this external parameter may change the admissible collection, the monotone measure, or the conditional aggregation rule. More generally, one may consider a parameterized system
$$\mathfrak S_t=(\cE_t,\mu_t,\bA_t),\quad t\in T,$$
and the corresponding two-variable generalized level surface
$$\Llambda_{\mathfrak S}(f;t,u) =
\sup\{\mu_t(E) : \sA_t(f|E)\ge u, E\in\cE_t\},$$ cf. \cite{BHKK}. 
Here $t$ is the external parameter selecting the information structure and the local evaluation mechanism, while $u$ is the level threshold used to test the aggregated support. Keeping these two variables separate is essential: diagonal expressions such as $u=t$ may be useful later, but they do not contain enough information to recover the full structural conditions behind localization.

Thus, the motivating question can be summarized as follows. Given a collection $\cE$ of admissible sets and a family $\bA$ of conditional aggregation rules, which structural conditions ensure that, for every $f\in\bF$, every localizing set $B\in\cE^0$, every positive level $u>0$, and every monotone measure $\mu$ on $\cE$, the identity
$$\Llambda_{\mu,\bA}(f\mI{B},u)
= \Llambda_{\mu,\bA}(f,u;B)$$
holds? The answer given in the paper is that this identity is governed precisely by the combination of monotonicity with respect to sets and the reduction property. The parameterized version presented in Section~\ref{sec:parameterization} then applies the same principle pointwise in the external parameter and provides the basis for later interpretations in terms of level-dependent evaluation curves, stochastic processes, and level-dependent integral functionals.

\section{Preliminaries}

We start with a nonempty set $X$. Whenever measurability is relevant, we fix an ambient $\sigma$-algebra $\Sigma\subseteq 2^X$. The role of $\Sigma$ is only to specify which functions and sets are measurable. It is assumed to be sufficiently rich for the collections considered below.
Throughout the paper, the only standing assumption on $\cE\subseteq\Sigma$ is that it is a \textit{paving}, that is,
$\emptyset\in\cE.$
We write $\cE^0:=\cE\setminus\{\emptyset\}$. Additional structural assumptions on $\cE$, such as closure under intersections, will be imposed explicitly when needed. The distinction between $\Sigma$ and $\cE$ is intentional. The $\sigma$-algebra $\Sigma$ is the ambient measurable structure, whereas $\cE$ is the working collection of admissible sets used for conditioning and for the suprema defining generalized level measures. Thus, $\Sigma$ may be much larger than $\cE$. For instance, $\cE$ may consist only of unions of blocks of a partition, events observable at a fixed time, scenario classes, or level-dependent feasible sets. In purely order-theoretic or finite discrete situations, we shall usually take $\Sigma=2^X$. Hence, statements involving $\cE=2^X$ are understood in the discrete measurable setting, where all subsets of $X$ are measurable.

The class of all nonnegative bounded $\Sigma$-measurable functions on $X$ is denoted by $\bF$. For $f,g\in\bF$, we write $f\le g$ on $X$ if $f(x)\le g(x)$ for every $x\in X$. The indicator function of a set $B\in\Sigma$ is denoted by $\mI{B}$, where
$$\mI{B}(x)=
\begin{cases}
1, & \text{if } x\in B,\\
0, & \text{otherwise}.
\end{cases}
$$
Since $\cE\subseteq\Sigma$, for every $f\in\bF$ and every $B\in\cE$ the masked function $f\mI{B}$ again belongs to $\bF$. 

By a monotone measure on a collection $\cE$ we mean a nondecreasing set function $\mu\colon\cE\to[0,\infty]$ such that $\mu(\emptyset)=0$ and $\mu(E)>0$ for at least one $E\in\cE$. Thus, whenever $C,D\in\cE$ and $C\subseteq D$, one has $\mu(C)\le\mu(D)$. If $X\in\cE$, one may replace the last nontriviality condition by $\mu(X)>0$. For background on monotone measures, capacities and non-additive set functions, see \cite{GrabischBook,Pap1995,WangKlir2009}.

The conditional aggregation tests used in this paper are represented by conditional aggregation operators. Such operators can be viewed as local nonlinear evaluation rules; the specific notion used below was introduced in \cite{BHHK}. For general background on nonlinear aggregation of several inputs, see \cite{Grabisch_book2}.

\begin{definition}
Let $\cE\subseteq\Sigma$ be a collection. A family
$$
\bA=\{\sA(\cdot| E):E\in\cE\}
$$
is called a \textit{family of conditional aggregation operators}, briefly an FCA, if for every $E\in\cE^0$ the mapping $\sA(\cdot|E)\colon\bF\to[0,\infty]$ satisfies:
\begin{enumerate}[label=(C\arabic*),leftmargin=1.5cm]
\item $\sA(f|E)\le\sA(g|E)$ for all $f,g\in\bF$ such that $f\le g$ on $E$;
\item $\sA(\mI{E^c}|E)=0$, where $E^c=X\setminus E$.
\end{enumerate}
\end{definition}

\noindent For the empty set, we use the convention $\sA(\cdot|\emptyset)=+\infty$.
With this convention, $\emptyset$ may be included in level suprema without causing exceptional cases: it always satisfies the level inequality, but its monotone measure is $\mu(\emptyset)=0$. Although the family $\bA$ depends on $\cE,$ we will not indicate it in our notation when there is no doubt of confusion. 

The monotonicity axiom (C1) immediately implies the following locality property, which will be used repeatedly below.

\begin{lemma}\label{lem:locality}
If $f=g$ on $E\in\cE$, then $\sA(f| E)=\sA(g| E)$.
\end{lemma}

\begin{proof}
For $E\ne\emptyset$, both $f\le g$ and $g\le f$ hold on $E$, and the assertion follows from \textup{(C1)}. For $E=\emptyset$ it follows from the convention above.
\end{proof}

\section{Structural conditions for localization}
\label{sec:structural-conditions}

The purpose of this section is to isolate two structural properties of the conditional aggregation rules which will be used later in the study of localization. The first property is monotonicity with respect to sets. It requires that enlarging the conditioning set cannot increase the corresponding local level value. This condition is natural for infimum-type, conjunctive, and worst-case evaluation mechanisms, but it is not automatic for arbitrary nonconstant families. The second property is a reduction property. It describes what happens when a signal is localized by multiplication with $\mI{B}$. Informally, it requires that after such a localization, no set which is not contained in $B$ can still produce a positive level value. This property may follow from a pointwise annihilator, but it may also be forced by the structure of the underlying collection itself.

We discuss these two properties separately. This makes it possible to distinguish the order-theoretic mechanism, encoded by monotonicity with respect to sets, from the localization mechanism, encoded by the reduction property.

\subsection{Monotonicity with respect to sets}\label{subsec:antitonicity}

\begin{definition}\rm
Let $\cE\subseteq \Sigma$ be a collection. We say that an FCA $\bA=\{\sA(\cdot|E): E\in\cE\}$ is \textit{nonincreasing with respect to sets} if
\begin{equation}\label{eq:nonincreasing_wrt_sets}C,D\in\cE^0, \,\, C\subseteq D\quad\Longrightarrow\quad
\sA(f|D)\le \sA(f|C)
\end{equation}
for every $f\in\bF$.
\end{definition}

\begin{remark}
Although the empty set does not cause a formal difficulty in the definition above, due to the convention $\sA(\cdot|\emptyset)=\infty$, we shall not regard it as relevant when discussing monotonicity with respect to sets. Thus, whenever structural properties of an FCA are considered, the substantive comparisons are those between nonempty  sets. This convention avoids degenerate or pathological interpretations involving $\emptyset$ and keeps the order-theoretic conditions tied to genuine conditioning sets.
\end{remark}

A trivial way to obtain such a family is to use the same conditional aggregation rule on all sets. For instance, if
$$\sA^{\inf}(f|E):=\inf_{x\in E} f(x), \quad E\in\cE^0,$$
then $\bA^{\inf} = \{\sA^{\inf}(\cdot|E) : E\in\cE\}$ is nonincreasing with respect to sets. However, the conditional aggregation rules assigned to different sets from the collection need not be the same. What is required is their compatibility with the inclusion ordering of the underlying sets.

The following example shows that it is not sufficient to know that the individual conditional aggregation rules are, in isolation, nonincreasing under enlargement of their domains. The way in which different rules are assigned to different sets also matters.

\begin{example}
Let $X=\{1,2,3\}$, let $\cE^0=\{\{1,2\},\{1,2,3\}\}$, and consider functions $f\colon X\to[0,1]$. Define a family $\bA$ by
\begin{align*}\sA(f|\{1,2,3\}) & = \min\{f(1),f(2),f(3)\}, \\\sA(f|\{1,2\}) & = \mathrm{W}(f(1),f(2)),
\end{align*}
where $\mathrm{W}$ is the Lukasiewicz t-norm,
$\mathrm{W}(a,b)=\max\{a+b-1,0\}.$ Although both the minimum t-norm and the Lukasiewicz t-norm are nonincreasing with respect to the addition of further arguments, the mixed family $\bA$ is not nonincreasing with respect to sets. Indeed, for $f=(0.9,0.9,0.85)$ we have $\{1,2\}\subseteq \{1,2,3\}$, but
$\sA(f|\{1,2,3\}) =
0.85 > 0.8 = \sA(f|\{1,2\}).$ Thus, the implication~\eqref{eq:nonincreasing_wrt_sets} fails.
\end{example}

The example indicates the correct principle. If larger sets are assigned stronger, that is, pointwise smaller, local evaluation mechanisms, then the family becomes nonincreasing with respect to sets. This can be made precise for the following ordered class of conjunctive rules, see \cite{KMPbook}. 

Let $\mathfrak T$ be a class of triangular norms on $[0,1]$. We write $T_1\preceq T_2$ if $T_1(a,b)\le T_2(a,b)$ for all $a,b\in[0,1]$. For a t-norm $T$, let $T^{(n)}$ denote its standard $n$-ary extension, defined recursively by
$$T^{(1)}(x_1)=x_1, \qquad T^{(n+1)}(x_1,\ldots,x_{n+1}) = T\bigl(T^{(n)}(x_1,\ldots,x_n),x_{n+1}\bigr).
$$ By associativity and commutativity, the value of $T^{(n)}$ is independent of the order of the arguments.

\begin{proposition}
Let $X$ be finite, let $\cE\subseteq 2^X$ be a paving, and let $\tau\colon \cE^0\to\mathfrak T$ be an assignment of t-norms to nonempty sets such that
$C\subseteq D \Longrightarrow \tau(D)\preceq \tau(C)$ for all $C,D\in\cE^0$. For $E\in\cE^0$, put
$$\sA_\tau(f| E):=\tau(E)^{(|E|)}\bigl((f(x))_{x\in E}\bigr),
\qquad f\colon X\to[0,1].$$
Then $\bA_\tau:=\{\sA_\tau(\cdot| E):E\in\cE\}$ is nonincreasing with respect to sets.
\end{proposition}

\begin{proof}
Let $C,D\in\cE^0$ with $C\subseteq D$. Put $T_D=\tau(D)$ and $T_C=\tau(C)$. By assumption, $T_D\preceq T_C$. This order is inherited by the finite extensions of the same arity: for every $m\in\mathbb N$ and every $(x_1,\ldots,x_m)\in[0,1]^m$,
$$
T_D^{(m)}(x_1,\ldots,x_m)\le T_C^{(m)}(x_1,\ldots,x_m).
$$
Indeed, this follows by induction on $m$ from the recursive definition of the finite extensions and the monotonicity of $T_C$.
Applying this with $m=|D|$ gives
$$
T_D^{(|D|)}\bigl((f(x))_{x\in D}\bigr)\le T_C^{(|D|)}\bigl((f(x))_{x\in D}\bigr).
$$
Since $C\subseteq D$ and every t-norm satisfies $T(a,b)\le a$ and $T(a,b)\le b$, adding further arguments cannot increase the value of the finite extension. Hence
$$
T_C^{(|D|)}\bigl((f(x))_{x\in D}\bigr)\le T_C^{(|C|)}\bigl((f(x))_{x\in C}\bigr).
$$
Combining the two inequalities yields
$\sA_\tau(f|D)\le\sA_\tau(f|C).$ Since $f$ was arbitrary, $\bA_\tau$ is nonincreasing with respect to sets.
\end{proof}

\begin{example}
On the standard scale of t-norms, one has the pointwise ordering
$$\mathrm{D}\preceq \mathrm{W}\preceq \Pi\preceq \mathrm{M},$$
where $\mathrm{D}$, $\mathrm{W}$, $\Pi$, and $\mathrm{M}$ denote, respectively, the drastic, Lukasiewicz, product, and minimum t-norms, cf. \cite{KMPbook}. Thus, the assignment $E\mapsto\tau(E)$ has to reverse the inclusion order of the admissible sets: whenever $C\subseteq D$, the rule assigned to $D$ must be pointwise smaller than the rule assigned to $C$. For instance, if $C\subseteq D$, then $\tau(D)=\mathrm{W}\preceq\mathrm{M}=\tau(C)$, and consequently $\sA_\tau(f|D)\le \sA_\tau(f|C)$ for every $f\colon X\to[0,1]$.
\end{example}

The preceding construction can be stated more generally. Suppose that we have a collection
$$\mathscr{B}=\{\bA^\alpha : \alpha\in I\} \,\,\,\textrm{with}\,\,\, \bA^\alpha = \{\sA^\alpha(\cdot|E) : E\in\cE^0\},$$
where $(I,\le_I)$ is a partially ordered set. Assume that each $\bA^\alpha$ is nonincreasing with respect to sets and that the collection $\mathscr{B}$ is ordered pointwise in the sense: for all $\alpha,\beta\in I$, 
$$\alpha\le_I \beta \quad\Longrightarrow\quad \sA^\alpha(f|E)\le \sA^\beta(f|E)$$
for every $E\in\cE^0$ and every $f\in\bF$.

\begin{proposition}\label{prop:selection}
Let $(I,\le_I)$ be a partially ordered set and let
$\mathscr{B}=\{\bA^\alpha:\alpha\in I\}$ be a pointwise ordered collection of local evaluation families, each nonincreasing with respect to sets. 
Let $\gamma\colon \cE^0\to I$ be a nonincreasing selector, that is,
$C\subseteq D \Longrightarrow
\gamma(D)\le_I\gamma(C).$
Define
$$\sA_\gamma(f|E) := \sA^{\gamma(E)}(f|E), \quad E\in\cE^0.
$$
Then $\bA_\gamma = \{\sA_\gamma(\cdot|E) : E\in\cE\}$ is nonincreasing with respect to sets.
\end{proposition}

\begin{proof}
Let $C,D\in\cE^0$ with $C\subseteq D$. Since $\gamma$ is nonincreasing, $\gamma(D)\le_I \gamma(C)$. By the pointwise ordering,
$$\sA^{\gamma(D)}(f|D)\le \sA^{\gamma(C)}(f|D).$$
Since $\bA^{\gamma(C)}$ is nonincreasing with respect to sets,
$$\sA^{\gamma(C)}(f|D) \le \sA^{\gamma(C)}(f|C).$$
Therefore, $\sA_\gamma(f|D)\le \sA_\gamma(f|C)$, which proves the claim.
\end{proof}

\begin{remark}
Proposition~\ref{prop:selection} gives a general selection principle for constructing nonconstant families of conditional aggregation rules. One may start with a collection of families which are ordered pointwise by the parameter order $\le_I$ and which are already known to be nonincreasing with respect to sets. Then, for each set $E$, one selects one family from this collection. The resulting family remains nonincreasing with respect to sets provided that the selector is compatible with inclusion in the order-reversing sense: larger sets must be assigned pointwise smaller local mechanisms.
\end{remark}

\subsection{Reduction property and its structural mechanisms}
\label{subsec:reduction-property}

We now discuss the reduction property used later in the localization theorem. Its role is to ensure that, after a signal is masked by $\mI{B}$, no set outside $B$ can still produce a positive level value. We also record several structural mechanisms which imply this property.

As in the case of monotonicity with respect to sets, the empty set does not play a substantive role. Indeed, $\emptyset\subseteq B$ for every $B\in\cE$, so the implication defining the reduction property is vacuous for $E=\emptyset$. We therefore formulate the condition only for nonempty conditioning sets, where the elimination mechanism is nontrivial.

\begin{definition}\label{def:property-R}
Let $\cE\subseteq \Sigma$ be a collection. We say that an FCA $\bA=\{\sA(\cdot|E): E\in\cE\}$ has the \emph{reduction property}, or simply property \textup{(R)}, if, for every $f\in\bF$ and every $B,E\in\cE^0$,
$$E\not\subseteq B \quad\Longrightarrow\quad \sA(f\mI{B}| E)=0.$$
\end{definition}

Equivalently,
$\sA(f\mI{B}| E)>0 \Longrightarrow E\subseteq B.$ Thus, property \textup{(R)} says that, after the function has been masked by $\mI{B}$, a strictly positive aggregated value can occur only on sets contained in $B$. In the language of localized level functions, masking the signal outside $B$ automatically removes all sets which are not contained in $B$.

\begin{lemma}\label{lem:R-positive-level}
If $\bA$ has property \textup{(R)}, then, for all $f\in\bF$, $B,E\in\cE^0$, and $u>0$,
$$\sA(f\mI{B}| E)\ge u \quad\Longrightarrow\quad E\subseteq B.$$
\end{lemma}

\begin{proof}
If $E\not\subseteq B$, then property \textup{(R)} gives $\sA(f\mI{B}| E)=0$, which is impossible when $\sA(f\mI{B}| E)\ge u>0$.
\end{proof}

\paragraph{The pointwise annihilator mechanism}
The simplest way to guarantee property \textup{(R)} is to require that a single zero value inside the conditioning set annihilates the local level value.

\begin{definition}\label{def:pointwise-annihilator}
We say that $0$ is a \emph{pointwise annihilator} of an FCA $\bA$ if, for every $E\in\cE^0$ and every $f\in\bF$,
$$\bigl(\exists x\in E:\ f(x)=0\bigr)\quad\Longrightarrow\quad \sA(f| E)=0.$$
\end{definition}

\begin{proposition}\label{prop:annihilator-implies-R}
If $0$ is a pointwise annihilator of $\bA$, then $\bA$ has property \textup{(R)}.
\end{proposition}

\begin{proof}
Let $f\in\bF$ and $B,E\in\cE^0$. Suppose that $E\not\subseteq B$. Then there exists $x\in E\setminus B$, and hence $(f\mI{B})(x)=0$. By the pointwise annihilator property, $\sA(f\mI{B}| E)=0$. Thus, $\bA$ has property \textup{(R)}.
\end{proof}

\begin{example}\label{ex:min-product-annihilator}
Assume that $X$ is finite. The conditional aggregation rules
$$\sA^{\min,p}(f|E)=\left(\min_{x\in E}f(x)\right)^p,\quad p>0,$$
have $0$ as a pointwise annihilator. Consequently, the corresponding FCA has property \textup{(R)}. On the scale $[0,1]$, the same observation applies to conditional aggregation rules induced by triangular norms. If $T$ is a t-norm and $T^{(|E|)}$ denotes its standard finite extension, then
$$\sA^T(f| E) = T^{(|E|)}\bigl((f(x))_{x\in E}\bigr), \qquad f\colon X\to[0,1],$$
has $0$ as a pointwise annihilator, since every t-norm satisfies $T(a,0)=0$. In particular, for $E=\{i_1,\ldots,i_m\}$,
$$\sA^{\Pi}(f|E)  = \prod_{\ell=1}^m f(i_\ell),\qquad
\sA^{\mathrm{W}}(f|E) = \max\left\{\sum_{\ell=1}^m f(i_\ell)-(m-1),0\right\}$$
both satisfy property \textup{(R)}.
\end{example}

The pointwise annihilator can also be imposed on an arbitrary FCA by a simple filtering construction.

\begin{proposition}\label{prop:annihilator-filtering}
Let $X$ be finite and let $\bA^0=\{\sA^0(\cdot|E) : E\in\cE\}$ be an FCA. For $E\in\cE^0$, define
$$\sA(f|E) = \begin{cases}\sA^0(f|E), & \textrm{if } f(x)>0 \text{ for every }x\in E,\\
0, & \text{otherwise.}
\end{cases}$$
Then $\bA=\{\sA(\cdot|E) : E\in\cE\}$ is an FCA, $0$ is a pointwise annihilator of $\bA$, and hence $\bA$ has property \textup{(R)}.
\end{proposition}

\begin{proof}
We verify \textup{(C1)}. Let $f,g\in\bF$, $E\in\cE^0$ and assume that $f\le g$ on $E$. If $f(x)=0$ for some $x\in E$, then $\sA(f|E)=0\le \sA(g|E)$. If $f(x)>0$ for every $x\in E$, then also $g(x)>0$ for every $x\in E$, and therefore
$$\sA(f|E) = \sA^0(f|E) \le \sA^0(g|E) = \sA(g|E),$$
by \textup{(C1)} for $\sA^0$. For \textup{(C2)}, note that $\mI{E^c}(x)=0$ for every $x\in E$, and hence $\sA(\mI{E^c}| E)=0$. Thus, $\bA$ is an FCA. The pointwise annihilator property is immediate from the definition of $\sA$, and property \textup{(R)} follows from Proposition~\ref{prop:annihilator-implies-R}.
\end{proof}

A concrete instance of Proposition~\ref{prop:annihilator-filtering} is
$$\sA(f|E)=\left(\max_{x\in E}f(x)\right)\prod_{x\in E}\mI{\{f(x)>0\}},$$
where $E$ is finite. The additional factor enforces the pointwise annihilator.

\paragraph{Point-separating collections}
The pointwise annihilator is sufficient for property \textup{(R)}, but it is not necessary in general. It becomes necessary whenever the collection is rich enough to separate single points inside the admissible sets.

\begin{definition}\label{def:point-separating}
We say that the collection $\cE\subseteq \Sigma$ is \emph{point-separating} if, for every $E\in\cE^0$ and every $x\in E$, there exists $B\in\cE^0$ such that
$x\notin B$ and $E\cap B=E\setminus\{x\}.$
\end{definition}

Thus, a collection $\cE$ is point-separating if every point of every nonempty set can be removed by intersecting it with another nonempty set. This property itself does not require $\cE$ to be closed under intersections. However, in the finite discrete setting, point-separation becomes very restrictive once closure under finite intersections and the condition $X\in\cE$ are imposed. A collection $\cE\subseteq\Sigma$ is said to be \textit{intersection-stable} if it satisfies
\begin{equation}\tag{E}\label{eq:intersection-stability}
C,D\in\cE \quad\Longrightarrow\quad C\cap D\in\cE.
\end{equation}

\begin{proposition}\label{prop:point-separating-finite-collapse}
Let $X=[n]$, $n\ge 2$, and let $\cE\subseteq 2^X$ be an intersection-stable and point-separating collection. 
If $X\in\cE$, then $\cE=2^X$.
\end{proposition}

\begin{proof}
We first show that $X\setminus\{x\}\in\cE$ for every $x\in X$. Fix $x\in X$. Applying point-separation to $E=X$ and $x\in X$, we obtain $B_x\in\cE^0$ such that $x\notin B_x$ and $X\cap B_x=X\setminus\{x\}$. Since $X\cap B_x=B_x$, it follows that $B_x=X\setminus\{x\}\in\cE$. Now let $E\subseteq X$. If $E=X$, then $E\in\cE$. If $E\neq X$, then
$$E = \bigcap_{x\in X\setminus E} \bigl(X\setminus\{x\}\bigr).$$
The intersection is finite, and each set $X\setminus\{x\}$ belongs to $\cE$. By \textup{(E)}, we obtain $E\in\cE$. Hence $\cE=2^X$.
\end{proof}

\begin{remark}
The assumption $X\in\cE$ cannot be omitted. For instance, if $X=\{1,2,3\}$, then
$$\cE=\{\emptyset,\{1\},\{2\},\{3\},\{1,2\},\{1,3\},\{2,3\}\}$$
is intersection-stable and point-separating, but $X\notin\cE$. Thus, $\cE\neq 2^X$.
\end{remark}

\begin{proposition}\label{prop:R-iff-annihilator-point-separating}
Assume that the collection $\cE\subseteq \Sigma$ is point-separating. Then an FCA $\bA=\{\sA(\cdot|E) : E\in\cE\}$ has property \textup{(R)} if and only if $0$ is a pointwise annihilator of $\bA$.
\end{proposition}

\begin{proof}
The implication from the pointwise annihilator to \textup{(R)} follows from Proposition~\ref{prop:annihilator-implies-R}. Conversely, assume that $\bA$ has property \textup{(R)}. Let $E\in\cE^0$, let $x_0\in E$, and let $f\in\bF$ be such that $f(x_0)=0$. By point-separation, there exists $B\in\cE^0$ such that $x_0\notin B$ and $E\cap B=E\setminus\{x_0\}$. In particular, $E\not\subseteq B$. Moreover, $f\mI{B}$ and $f$ coincide on $E$: outside $B$, the only point of $E$ is $x_0$, and $f(x_0)=0$. Hence, by Lemma~\ref{lem:locality},
$$\sA(f|E)=\sA(f\mI{B}|E).$$ 
Since $E\not\subseteq B$, property \textup{(R)} gives $\sA(f\mI{B}|E)=0$. Therefore $\sA(f|E)=0$, and so $0$ is a pointwise annihilator of $\bA$.
\end{proof}

\paragraph{The role of the underlying collection}
If $\cE$ is not point-separating, property \textup{(R)} may still hold without a pointwise annihilator. In such cases the structure of $\cE$ itself may force the required reduction.

\begin{proposition}\label{prop:disjoint-collection-R}
Assume that the nonempty members of $\cE$ are pairwise disjoint, that is, $C\cap D=\emptyset$ for all $C,D\in\cE^0$ with $C\neq D$. Then every FCA $\bA=\{\sA(\cdot| E):E\in\cE\}$ has property \textup{(R)}.
\end{proposition}

\begin{proof}
Let $B,E\in\cE^0$ and suppose that $E\not\subseteq B$. Since the nonempty members of $\cE$ are pairwise disjoint, this implies $E\cap B=\emptyset$. Therefore $f\mI{B}=0$ on $E$. By Lemma~\ref{lem:locality} and \textup{(C2)}, $\sA(f\mI{B}|E) = \sA(\mI{E^c}|E) = 0$. Thus, $\bA$ has property \textup{(R)}.
\end{proof}

More interesting examples occur when the admissible sets are generated by blocks. In the following construction the blocks are indexed, and this indexing determines the order in which the block values are inserted into the outer evaluation. Thus no symmetry of the outer evaluation is required. If one wants the construction to be independent of the labelling of the blocks, then symmetry of the outer evaluation should be imposed.

\begin{proposition}\label{prop:block-mechanism-R}
Let $X=[n]$ and let $\mathcal P=\{D_1,\ldots,D_k\}$ be a partition of $X$. Put
$$\cE = \left\{\bigcup_{j\in J}D_j:J\subseteq[k]\right\}.$$
Let $\widetilde{\bA}=\{\widetilde{\sA}(\cdot|D_j) : j\in[k]\}$ be a family of conditional aggregation rules on the blocks. Let $\Gamma_m\colon [0,\infty]^m\to[0,\infty]$, $m=1,\ldots,k$, be nondecreasing in every coordinate and have $0$ as an annihilator, that is, if
$z_\ell=0$ for some $\ell$, then $\Gamma_m(z_1,\ldots,z_m)=0.$
For $E\in\cE^0$, write
$$
J(E)=\{j\in[k]:D_j\subseteq E\}=\{j_1<\cdots<j_m\},
$$
and define
$$
\sA(f|E)
=
\Gamma_m\left(
\widetilde{\sA}(f|D_{j_1}),\ldots,\widetilde{\sA}(f|D_{j_m})
\right).
$$
Then $\bA=\{\sA(\cdot|E) : E\in\cE\}$ is an FCA and has property \textup{(R)}.
\end{proposition}

\begin{proof}
Let $E\in\cE^0$ and write $J(E)=\{j_1<\cdots<j_m\}$. Let $f,g\in\bF$ and suppose that $f\le g$ on $E$. Then $f\le g$ on every block $D_{j_r}\subseteq E$. Hence
$$\widetilde{\sA}(f|D_{j_r})\le \widetilde{\sA}(g|D_{j_r}),
\qquad r=1,\ldots,m,$$
by \textup{(C1)} for $\widetilde{\sA}$. Since $\Gamma_m$ is nondecreasing in every coordinate, we obtain
$\sA(f|E)\le \sA(g|E).$ Thus, \textup{(C1)} holds.

For \textup{(C2)}, take $f=\mI{E^c}$. If $D_{j_r}\subseteq E$, then $\mI{E^c}=0$ on $D_{j_r}$. By Lemma~\ref{lem:locality} and \textup{(C2)} for $\widetilde{\sA}$, we have
$$\widetilde{\sA}(\mI{E^c}|D_{j_r})=0, \qquad r=1,\ldots,m.$$
Thus, all inputs of $\Gamma_m$ are equal to $0$, and the annihilator property gives $\sA(\mI{E^c}| E)=0.$ Hence, $\bA$ is an FCA.

It remains to prove \textup{(R)}. Let $B,E\in\cE^0$ and assume that $E\not\subseteq B$. Since both $B$ and $E$ are unions of blocks, there is a block $D_{j_0}$ such that
$D_{j_0}\subseteq E$ and $D_{j_0}\cap B=\emptyset$. 
Consequently, $f\mI{B}=0$ on $D_{j_0}$. By Lemma~\ref{lem:locality} and \textup{(C2)} for $\widetilde{\sA}$, we get
$$
\widetilde{\sA}(f\mI{B}|D_{j_0})=0.
$$
Thus, one of the inputs entering the outer evaluation is equal to $0$. Since $0$ is an annihilator of $\Gamma_m$, we obtain $\sA(f\mI{B}|E)=0.$ Therefore, $\bA$ has property \textup{(R)}.
\end{proof}

\begin{example}\label{ex:block-averages}
Let $X=[n]$ and let $\mathcal P=\{D_1,\ldots,D_k\}$ be a partition of $X$. Put
$$\cE = \left\{\bigcup_{j\in J}D_j:J\subseteq[k]\right\},$$
and for $E\in\cE^0$ define
$$\sA^{\rm block}(f|E) = \min\left\{\frac1{|D_j|}\sum_{i\in D_j}f(i) : D_j\subseteq E\right\}.$$ Then $\bA^{\rm block}$ is an FCA and has property \textup{(R)}. Moreover, it is nonincreasing with respect to sets: if $C\subseteq D$, then the minimum defining $\sA^{\rm block}(f|D)$ is taken over at least as many block averages as the minimum defining $\sA^{\rm block}(f|C)$.

If at least one block $D_j$ contains more than one element, then $0$ is not a pointwise annihilator of $\bA^{\rm block}$. Indeed, take $E=D_j$, choose $x_0\in D_j$, and define $f(x_0)=0$ and $f(x)=1$ for every $x\in D_j\setminus\{x_0\}$. Then
$$\sA^{\rm block}(f|E) = \frac{|D_j|-1}{|D_j|} > 0.$$
Thus, property \textup{(R)} may hold even when the pointwise annihilator fails.
\end{example}

The annihilator assumption on the outer evaluation is essential in this mechanism.

\begin{example}\label{ex:block-failure-max}
Let $X=\{1,2,3,4\}$, $D_1=\{1,2\}$, $D_2=\{3,4\}$, and let $\cE=\{\emptyset,D_1,D_2,X\}$. Define
$$\sA^{\max\textrm{-block}}(f|E) = \max\left\{\frac1{|D_j|}\sum_{i\in D_j}f(i) : D_j\subseteq E \right\}.$$
This is an FCA, but it does not have property \textup{(R)}. Indeed, take
$$B=D_1,\qquad E=X,\qquad f=(2,3,4,5).$$
Then $f\mI{B}=(2,3,0,0)$, and hence
$$\frac1{|D_1|}\sum_{i\in D_1}(f\mI{B})(i)=\frac52, \qquad \frac1{|D_2|}\sum_{i\in D_2}(f\mI{B})(i)=0.$$
Therefore,
$$\sA^{\max\textrm{-block}}(f\mI{B}| X)=\frac52>0,$$
although $X\not\subseteq B$. Thus, property \textup{(R)} fails.
\end{example}

\paragraph{Almost-everywhere versions}
In measure-theoretic applications, exact pointwise annihilation is often too strong. A zero on a null set should usually be ignored. This leads to an almost-everywhere version of property \textup{(R)}.

\begin{remark}\label{rem:essential-infimum-R}
Let $(X,\Sigma,\lambda)$ be a Lebesgue measure space and let
$$\sA^{\rm essinf}(f|E) = \operatorname*{ess\,inf}_{x\in E} f(x).$$
Then $\sA^{\rm essinf}$ is nonincreasing with respect to sets. It does not have a pointwise annihilator, because values on $\lambda$-null sets are ignored. However, for every $u>0$,
$$\sA^{\rm essinf}(f\mI{B}|E)\ge u \quad\Longrightarrow\quad \lambda(E\setminus B)=0.$$
Indeed, if $\lambda(E\setminus B)>0$, then $f\mI{B}=0$ on a subset of $E$ of positive $\lambda$-measure, and therefore
$\sA^{\rm essinf}(f\mI{B}| E)=0.$
Thus, the exact reduction property is replaced by the almost-everywhere reduction
$$\sA^{\rm essinf}(f\mI{B}|E)\ge u>0 \quad\Longrightarrow\quad E\subseteq B \quad \text{modulo }\lambda\text{-null sets}.$$
Consequently, the order-theoretic localization theorem has a natural measure-theoretic analogue whenever the underlying sets and monotone measures are considered modulo null sets, or whenever the relevant set functions are insensitive to null modifications. This is in line with the treatment of null-additive set functions in non-additive measure theory \cite{Pap1995}.
\end{remark}

\begin{remark}
The reduction property may arise from different mechanisms. It is implied by the pointwise annihilator, it is equivalent to the pointwise annihilator for point-separating collections, it is automatic for pairwise disjoint collections, and it can be produced by block-generated structures through an outer evaluation with $0$ as an annihilator. In measure-theoretic settings, exact reduction is naturally replaced by reduction modulo null sets, as in the case of the essential infimum.
\end{remark}

\section{The localization theorem}\label{subsec:localization-theorem}

We now return to generalized level measures and show how the two structural conditions introduced above interact. The aim is to compare two ways of localizing the level analysis to a fixed set $B$: one may either first mask the signal, replacing $f$ by $f\mI{B}$, or one may keep the original signal $f$ and localize the admissible sets in the supremum by intersection with $B$. The localization theorem below identifies conditions under which these two procedures are equivalent: monotonicity with respect to sets controls the passage from larger admissible sets to their localized parts, while the reduction property ensures that the masked signal cannot produce positive level values on sets outside $B$.

For an FCA $\bA$ and a monotone measure $\mu$ on $\cE$, define the \textit{generalized level measure}
\begin{equation}\label{eq:generalized-level}
\Llambda_{\mu,\bA}(f,u)
:= \sup\{\mu(E):\sA(f|E)\ge u,\ E\in\cE\}, \quad u\ge 0,
\end{equation} introduced in \cite{BHHK}. For $B\in\cE^0$ define its localized version by
\begin{equation}\label{eq:localized-level}
\Llambda_{\mu,\bA}(f,u;B)
:= \sup\{\mu(E\cap B):\sA(f|E)\ge u,\ E\in\cE\}, \quad u\ge 0.
\end{equation}

\begin{tw}\label{thm:main-nonparametric}
Let $\cE$ be an intersection-stable collection and let $\bA$ be an FCA on $\cE$. The following conditions are equivalent:
\begin{enumerate}
\item[\textup{(i)}] for every $f\in\bF$, every $B\in\cE^0$, every $u>0$ and every monotone measure $\mu$ on $\cE$, one has
\begin{equation}\label{eq:S-main}
\Llambda_{\mu,\bA}(f\mI{B},u)=\Llambda_{\mu,\bA}(f,u;B);
\end{equation}
\item[\textup{(ii)}] the FCA $\bA$ is nonincreasing with respect to sets and has property \textup{(R)}.
\end{enumerate}
\end{tw}

\begin{proof}
(ii)$\Rightarrow$(i) Fix $f\in\bF$, $B\in\cE^0$, $u>0$ and a monotone measure $\mu$ on $\cE$. We first reduce the left-hand side. If $E\in\cE$ satisfies $\sA(f\mI{B}|E)\ge u$, then either $E=\emptyset$, in which case $E\subseteq B$, or $E\in\cE^0$, in which case Lemma~\ref{lem:R-positive-level} gives $E\subseteq B$. By Lemma~\ref{lem:locality}, $\sA(f\mI{B}|E)=\sA(f|E)$ whenever $E\subseteq B$. Hence,
\begin{equation}\label{eq:L-reduced}
\Llambda_{\mu,\bA}(f\mI{B},u)=\sup\{\mu(E):E\subseteq B,\ \sA(f|E)\ge u,\ E\in\cE\}.
\end{equation}

Let $E\in\cE$ satisfy $\sA(f|E)\ge u$ and put $F:=E\cap B$. Since $\cE$ is intersection-stable, $F\in\cE$. Since $F\subseteq E$, monotonicity with respect to sets gives
$\sA(f|E)\le\sA(f|F),$
where the case $F=\emptyset$ is covered by the convention $\sA(\cdot|\emptyset)=\infty$. Hence $\sA(f|F)\ge u$. Since $F\subseteq B$, the set $F$ contributes to the reduced supremum~\eqref{eq:L-reduced}, and therefore
$$\mu(E\cap B)=\mu(F)\le \Llambda_{\mu,\bA}(f\mI{B},u).$$
Taking the supremum over all $E\in\cE$ satisfying $\sA(f|E)\ge u$ gives
$$\Llambda_{\mu,\bA}(f,u;B)\le \Llambda_{\mu,\bA}(f\mI{B},u).$$

Conversely, let $F\in\cE$ contribute to the reduced supremum~\eqref{eq:L-reduced}. Choose $E:=F$. Then $E\cap B=F$ and $\sA(f|E)\ge u$, so
$$
\mu(F)\le \Llambda_{\mu,\bA}(f,u;B).
$$
Taking the supremum in \eqref{eq:L-reduced} gives
$$
\Llambda_{\mu,\bA}(f\mI{B},u)\le \Llambda_{\mu,\bA}(f,u;B).
$$
Thus, $\Llambda_{\mu,\bA}(f\mI{B},u)=\Llambda_{\mu,\bA}(f,u;B)$.\medskip 

(i)$\Rightarrow$(ii) First we prove property \textup{(R)}. Suppose that, for some $f\in\bF$, $B\in\cE^0$ and $E_0\in\cE^0$, we have
$$\sA(f\mI{B}| E_0)>0
\quad\text{and}\quad
E_0\not\subseteq B.
$$
Choose $u>0$ such that $\sA(f\mI{B}| E_0)\ge u$; if the value is infinite, take any $u>0$. Define
$$
\mu_B(F)=
\begin{cases}
0,&F\subseteq B,\\
1,&F\not\subseteq B,
\end{cases}
\qquad F\in\cE.
$$
This is a monotone measure on $\cE$. Since $E_0$ contributes to the left-hand side of \eqref{eq:S-main} and $\mu_B(E_0)=1$, the left-hand supremum is equal to $1$. On the right-hand side, every set has the form $F\cap B\subseteq B$, hence $\mu_B(F\cap B)=0$. The right-hand supremum is therefore $0$, a contradiction. Thus, no such $E_0$ exists, and property \textup{(R)} holds.

It remains to prove monotonicity. Let $C,D\in\cE^0$ with $C\subseteq D$, and fix $f\in\bF$. We show that $\sA(f| D)\le\sA(f| C)$. If $\sA(f| D)=0$, there is nothing to prove. If $0<\sA(f| D)<\infty$, put $u:=\sA(f| D)$. If $\sA(f| D)=\infty$, let $u>0$ be arbitrary; the following argument then gives $\sA(f| C)\ge u$ for every $u>0$, which implies $\sA(f| C)=\infty$.

For such a finite positive $u$, or for an arbitrary $u>0$ in the infinite case, define
$$\mu_C(F):=\mI{\{C\subseteq F\}},
\quad F\in\cE.$$
This is a monotone measure on $\cE$ because $C\ne\emptyset$. Apply \eqref{eq:S-main} with $B=C$. Since $D$ contributes to the right-hand side, $\sA(f|D)\ge u$, and $D\cap C=C$, the right-hand supremum is equal to $1$. Hence the left-hand supremum is equal to $1$. As $\mu_C$ is $\{0,1\}$-valued, there exists $E\in\cE^0$ such that
$\sA(f\mI{C}| E)\ge u$ and $C\subseteq E.$ By Lemma~\ref{lem:R-positive-level}, applied with $B=C$, we get $E\subseteq C$. Therefore $E=C$. By Lemma~\ref{lem:locality},
$$\sA(f|C)=\sA(f\mI{C}|C)\ge u.$$
In the finite positive case this yields $\sA(f|C)\ge\sA(f|D)$. In the infinite case it holds for every $u>0$, hence $\sA(f|C)=\infty$. Thus, $\bA$ is nonincreasing with respect to sets.
\end{proof}

In point-separating intersection-stable collections, the reduction property admits a simpler interpretation. In this case, property \textup{(R)} is equivalent to the existence of a pointwise annihilator. Consequently, once monotonicity with respect to sets is assumed, the localization identity can be characterized directly in terms of the pointwise behaviour of the conditional aggregation rules.

\begin{corollary}
\label{cor:point-separating-S-annihilator}
Assume that $\cE$ is intersection-stable and point-separating, and that $\bA$ is nonincreasing with respect to sets. Then the following assertions are equivalent:
\begin{enumerate}[label=\textup{(\roman*)},leftmargin=1.2cm]
\item $0$ is a pointwise annihilator of $\bA$;
\item $\bA$ has property \textup{(R)};
\item the localization identity~\eqref{eq:S-main} holds for every $f\in\bF$, every $B\in\cE^0$, every $u>0$, and every monotone measure $\mu$ on $\cE$.
\end{enumerate}
\end{corollary}

\begin{proof}
The equivalence of \textup{(i)} and \textup{(ii)} follows from Proposition~\ref{prop:R-iff-annihilator-point-separating}. Since $\bA$ is assumed to be nonincreasing with respect to sets, the equivalence of \textup{(ii)} and \textup{(iii)} follows from Theorem~\ref{thm:main-nonparametric}.
\end{proof}

\begin{remark}
The restriction $u>0$ in the localization identity is essential. At level $u=0$, all admissible sets satisfy the level condition because the local evaluation values are nonnegative. Thus, masking by $\mI{B}$ does not exclude sets outside $B$ at the zero level, and the reduction property cannot force the desired localization.
\end{remark}

\section{Localization in parameterized systems}\label{sec:parameterization}

We now pass from a single collection and a single local evaluation system to a parameterized family of such objects, cf.~\cite{BHKK}. The parameter may represent time, scale, a clinical regime, a confidence level, or any other external variable which changes the observable sets, the underlying monotone measure, or the local evaluation mechanism. The important point is that this external parameter must be kept separate from the level variable used in the generalized level measure.

Let $T$ be a non-empty parameter set. A parameterized system is a family of triples
$$\mathfrak S_t=(\cE_t,\mu_t,\bA_t),
\quad t\in T,$$
where each $\cE_t$ is an  intersection-stable collection, each $\mu_t$ is a monotone measure on $\cE_t$, and
$$\bA_t=\{\sA_t(\cdot| E):E\in\cE_t\}$$
is an FCA on $\cE_t$. We write
$$\bE=(\cE_t)_{t\in T},
\qquad \bmu=(\mu_t)_{t\in T},
\qquad \bpeA=(\bA_t)_{t\in T}.$$
As before, the convention $\sA_t(\cdot|\emptyset)=\infty$ is used when $\emptyset\in\cE_t$.
For $f\in\bF$, $t\in T$ and $u\ge0$, define the parameterized generalized level surface by
\begin{equation}\label{eq:param-surface}
\Llambda_{\mathfrak S}(f;t,u)
:=
\sup\{\mu_t(E):\sA_t(f| E)\ge u,
\ E\in\cE_t\}.
\end{equation}
For $B\in\cE_t^0$ define the corresponding localized surface by
\begin{equation}\label{eq:param-local-surface}
\Llambda_{\mathfrak S}(f;t,u;B)
:=
\sup\{\mu_t(E\cap B):\sA_t(f| E)\ge u,
\ E\in\cE_t\}.
\end{equation}
The level surface may be defined also at $u=0$. However, the localization identity below is stated for $u>0$, since the reduction property eliminates sets outside $B$ only at positive levels.

\begin{definition}\label{def:param-antitone-R}
Fix $t\in T$. We say that the parameterized FCA $\bpeA$
\begin{enumerate}[label=(\roman*),leftmargin=1.2cm]
\item  is \emph{nonincreasing with respect to sets at $t$} if, for all $C,D\in\cE_t$ with $C\subseteq D$ and $f\in\bF$,
$$\sA_t(f| D)\le\sA_t(f| C);$$
\item has the \emph{reduction property at $t$}, denoted \textup{(R$_t$)}, if, for all $f\in\bF$ and $B,E\in\cE_t^0$,
$$E\not\subseteq B
\quad\Longrightarrow\quad
\sA_t(f\mI{B}| E)=0.$$
\end{enumerate}
\end{definition}

Thus, the two structural conditions from the non-parameterized setting are required only after the parameter has been fixed. No compatibility between different parameter values is needed for the pointwise localization statement. 

\begin{tw}\label{thm:parametric}
Let $t\in T$ be fixed. The following conditions are equivalent:
\begin{enumerate}[label=(\roman*),leftmargin=1.2cm]
\item for every $f\in\bF$, every $B\in\cE_t^0$, every $u>0$ and every monotone measure $\nu_t$ on $\cE_t$,
\begin{equation}\label{eq:param-S}
\begin{aligned}
\sup\{\nu_t(E):\sA_t(f\mI{B}| E)\ge u, E\in\cE_t\}
&=\\
\sup\{\nu_t(E\cap B):\sA_t(f| E)\ge u, E\in\cE_t\}.
\end{aligned}
\end{equation}
\item $\bpeA$ is nonincreasing with respect to sets at $t$ and has property \textup{(R$_t$)}.
\end{enumerate}
Consequently, if condition \textup{(ii)} holds for every $t\in T$, then, for every choice of localizing sets $B_t\in\cE_t^0$,
\begin{equation}\label{eq:param-localization-all}
\Llambda_{\mathfrak S}(f\mI{B_t};t,u)
=
\Llambda_{\mathfrak S}(f;t,u;B_t)
\end{equation}
holds for every $t\in T$, every $u>0$ and every $f\in\bF$.
\end{tw}

\begin{proof}
For fixed $t\in T$, apply Theorem~\ref{thm:main-nonparametric} to the intersection-stable collection $\cE_t$ and the FCA $\bA_t$. This gives the equivalence of \textup{(i)} and \textup{(ii)}. The final assertion follows by taking $\nu_t=\mu_t$ and $B=B_t$.
\end{proof}

\begin{remark}
The parameter $t$ and the level $u$ play different roles. The parameter $t$ selects the underlying collection, the monotone measure and the local evaluation system, while $u$ is the threshold at which the level analysis is performed. If the theory were formulated only on the diagonal $u=t$, then equality of the two suprema would test localization only at that particular level. It would not recover the full monotonicity condition
$$\sA_t(f| D)\le\sA_t(f| C)$$
for arbitrary values of $\sA_t(f|D)$, because the proof of monotonicity in Theorem~\ref{thm:main-nonparametric} tests the localization identity at the level $u=\sA_t(f|D)$ whenever this value is finite and positive. This level need not coincide with the parameter $t$. The two-variable formulation $(t,u)$ is therefore the natural one; the diagonal case is only a specialization.
\end{remark}

\paragraph{Context-dependent knowledge evaluation}
The parameterized framework can be interpreted as a formal model for context-dependent evaluation in knowledge-based systems, in the broad sense of structured symbolic or evidence-based information processing; see, e.g., \cite{BrachmanLevesque2004,vanHarmelen2008}. We use this interpretation in a structural sense. The aim is not to model the full inference mechanism of an arbitrary knowledge-based system, but rather the common situation in which a system evaluates a query, hypothesis, decision alternative or diagnostic state by selecting admissible collections of information units and aggregating their local degrees of support.

Let $X$ denote the universe of elementary information units. Depending on the system, elements of $X$ may represent facts, rules, symptoms, cases, documents, expert statements, criteria or elementary pieces of evidence. A nonnegative function $f\in\bF$ assigns to each unit a score, such as relevance to a query, reliability, compatibility with a hypothesis, degree of activation, confidence, or local support value. A parameter $t\in T$ may represent a query, decision context, resolution level, user profile, time instant or regime of admissibility.

The collection $\cE_t$ represents the admissible knowledge contexts available at parameter value $t$. These contexts are not arbitrary subsets of $X$. They may be constrained by logical consistency, semantic compatibility, observability, modular structure, expert-source restrictions, scenario constraints, or the architecture of a rule base. Thus, a set $E\in\cE_t$ should be read as an information context which is admissible for the current evaluation. The monotone measure $\mu_t$ assigns an informational weight to such a context. This weight may represent support, coverage, importance, confidence, relevance, or a non-additive capacity of the information carried by $E$.

The conditional aggregation operator $\sA_t(\cdot|E)$ evaluates the quality of the score function inside the context $E$. Hence the level condition
$$\sA_t(f|E)\ge u$$
means that the context $E$ passes the local acceptability test at level $u$. The generalized level measure
$$\Llambda_{\mathfrak S}(f;t,u)
=\sup\{\mu_t(E):\sA_t(f|E)\ge u,\ E\in\cE_t\}$$
therefore has a direct selection meaning: it is the maximal informational weight of an admissible context whose aggregated support reaches the required level. In this sense, the generalized level measure is the value of a level-constrained evidence selection problem over admissible knowledge contexts.

Localization enters when the current task makes only a part $B$ of the knowledge base relevant. The set $B$ may be determined by a query filter, a clinical regime, a user profile, a semantic restriction, an active rule module, a trusted source set, or a context window. There are then two natural ways of performing the evaluation. One may first restrict the score function to the relevant part of the knowledge base, replacing $f$ by $f\mI{B}$, and then solve the level-constrained selection problem. Alternatively, one may evaluate contexts in the original system and localize only their informational weights by replacing $E$ with $E\cap B$.

The localization identity
$$\Llambda_{\mathfrak S}(f\mI{B};t,u)
=\Llambda_{\mathfrak S}(f;t,u;B)$$
is therefore not merely a formal equality. It expresses a consistency requirement for context filtering: pre-filtering the knowledge base and post-filtering the selected evidence contexts lead to the same level-based evaluation. In a knowledge-based system, this means that the final score is invariant under the order in which contextual restriction and level-based evidence selection are performed.

Theorem~\ref{thm:parametric} shows that this consistency is governed by two structural mechanisms. Monotonicity with respect to sets says that enlarging an admissible context cannot improve the local lower-level evaluation. This is natural for conjunctive, infimum-type, worst-case and robust support rules, where adding further pieces of information may introduce additional constraints. The reduction property says that after the score has been masked outside $B$, no context which is not contained in $B$ can still pass a positive level test. In knowledge-based terminology, information units outside the active context cannot contribute to a positive localized support value.

This distinction is important. The reduction property need not be a pointwise property of individual information units. In modular knowledge bases, it may be produced by the structure of admissible contexts. Suppose, for example, that the knowledge base is decomposed into modules
$$\mathcal P_t=\{D_{t,1},\ldots,D_{t,k_t}\},$$
where the modules represent rule groups, expert sources, symptom clusters, document classes, scenario blocks, or other semantically coherent information units. If $\cE_t$ consists of unions of such modules, one may first evaluate the score inside each module and then combine the module scores by an outer aggregation with $0$ as an annihilator. If a module lies entirely outside the active context $B$, masking by $\mI{B}$ gives a zero module score, and the outer aggregation propagates this zero to the whole context. Thus, localization consistency may be enforced at the modular level even when a single zero value inside a module does not annihilate the local module evaluation.

The theorem can therefore be read as a structural criterion for reliable context-dependent knowledge evaluation. It identifies when a generalized level-based score is stable under query restriction, removal of irrelevant evidence, or activation of a context window. Such stability is essential in systems where admissible information states are constrained by semantic, logical, modular, or observational structure, and where the evaluation of a hypothesis or decision alternative is based on non-additive aggregation of local support.

\paragraph{A localization-constrained selection problem}
The preceding interpretation leads to a simple selection problem. Let $\mathcal U$ be a nonempty set of hypotheses, configurations or decision alternatives, and suppose that each $a\in\mathcal U$ is represented by a structured support signal $f_a\in\bF$. For a fixed parameter $t$, the value
$$V_t(a,u):=\Llambda_{\mathfrak S}(f_a;t,u)$$
can be read as the largest informational weight of an admissible context on which alternative $a$ reaches the aggregated support level $u$. If only a localized part $B_t\in\cE_t^0$ of the knowledge structure is relevant, the corresponding localized score is
$$V_t^B(a,u):=\Llambda_{\mathfrak S}(f_a;t,u;B_t).$$
The localization theorem gives structural conditions under which the two natural procedures, masking the support signal first or localizing the admissible contexts in the evaluation, lead to the same value. This matters when irrelevant, unreliable, delayed or contextually excluded information units should not affect the final comparison of hypotheses or alternatives.

Fixing a positive level $u$ and a class $\mathcal B_t\subseteq\cE_t^0$ of admissible localization contexts, one obtains
$$\max_{a\in\mathcal U,\,B\in\mathcal B_t}\Llambda_{\mathfrak S}(f_a;t,u;B).$$
Under the assumptions of Theorem~\ref{thm:parametric}, this problem can equivalently be written as
$$\max_{a\in\mathcal U,\,B\in\mathcal B_t}\Llambda_{\mathfrak S}(f_a\mI{B};t,u),$$
so that the selection can be performed either on localized admissible contexts or on masked support signals. The equality is not merely formal: without monotonicity with respect to sets or without the reduction property, the two formulations may rank alternatives differently. Thus, the localization theorem provides a consistency principle for level-based knowledge evaluation with context-dependent feasible information.

\paragraph{Local and global localization sets}

The parameterized localization theorem is local in the parameter. At each parameter value, one may choose a different localizing set $B_t\in\cE_t^0$. If $\bpeA$ is nonincreasing with respect to sets at $t$ and has property \textup{(R$_t$)} for every $t\in T$, then
\begin{equation}\label{eq:local-Btheta}
\Llambda_{\mathfrak S}(f\mI{B_t};t,u)
=
\Llambda_{\mathfrak S}(f;t,u;B_t),
\quad B_t\in\cE_t^0,
\quad u>0.
\end{equation}

If one wants to localize by one fixed set $B$ for all parameter values, then $B$ must belong to all the corresponding collections. Hence the natural domain of global localization is
$$\cE_T^0:=\bigcap_{t\in T}\cE_t^0.$$
If $B\in\cE_T^0$, then \eqref{eq:param-localization-all} gives
\begin{equation}\label{eq:global-B}
\Llambda_{\mathfrak S}(f\mI{B};t,u)
=
\Llambda_{\mathfrak S}(f;t,u;B),
\quad t\in T,
\quad u>0.
\end{equation}
If $\cE_T^0=\emptyset$, then no non-trivial global localization set exists. This is not a defect of the theorem but a structural feature of the parameterized information structures.

\paragraph{Interpretation in stochastic processes}
The framework also admits a natural stochastic interpretation. Let
$$(\Omega,\cF,(\cF_t)_{t\ge0},\mathsf P)$$
be a filtered probability space \cite{Kallenberg} and let $Y=(Y_t)_{t\ge0}$ be a non-negative adapted process. For each time $t$, the sets in $\cF_t$ represent events observable at time $t$. Thus one may take, for instance,
$$
X=\Omega,\qquad \cE_t\subseteq\cF_t,\qquad \mu_t=\mathsf P|_{\cE_t}.
$$
The choice $\cE_t=\cF_t$ corresponds to the full information available at time $t$, while smaller collections $\cE_t$ may represent restricted observable regimes, scenario classes, or coarser information structures.

A first canonical choice is the essential-infimum conditional aggregation rule
$$\sA^{\textrm{essinf}}_t(Z|E)=\essinf_{\omega\in E} Z(\omega),
\quad E\in\cF_t.$$
With the usual convention modulo null sets, the generalized level measure reduces to the ordinary level measure whenever $Z$ is $\cF_t$-measurable:
\begin{equation}\label{eq:stochastic-survival}
\sup\{\mathsf P(E) : 
\essinf_E Z\ge u, E\in\cF_t\} =
\mathsf P(Z\ge u).
\end{equation}
Indeed, the condition $\essinf_E Z\ge u$ is equivalent to $E\subseteq\{Z\ge u\}$ up to a $\mathsf P$-null set; since $\{Z\ge u\}\in\cF_t$, the largest observable choice is the event $\{Z\ge u\}$ itself. Taking $Z=Y_t$ gives
$$\sup\{\mathsf P(E) : 
\essinf_E Y_t\ge u, E\in\cF_t\}
= \mathsf P(Y_t\ge u).$$
If $B\in\cF_t$, then for $u>0$,
\begin{equation}\label{eq:stochastic-local-survival}
\sup\{\mathsf P(E) : 
\essinf_E(Y_t\mI{B})\ge u, E\in\cF_t\}
=
\mathsf P(\{Y_t\ge u\}\cap B).
\end{equation}
Thus, the localization theorem generalizes the elementary fact that the probability of a level event after restriction to an event $B$ is obtained by intersecting the level set with $B$.

A less extreme interpretation is obtained from conditional means over information atoms. Suppose that, for a fixed $t$, the collection $\cE_t$ is generated by a finite $\cF_t$-measurable partition
$\mathcal P_t=\{D_{t,1},\ldots,D_{t,k_t}\}$ with $\mathsf P(D_{t,j})>0$, namely
$$
\cE_t=
\left\{
\bigcup_{j\in J}D_{t,j}:J\subseteq[k_t]
\right\}.
$$
For a non-negative integrable random variable $Z$, define
$$\sA^{\textrm{atom}}_t(Z|E)
:= \min\left\{
\frac{1}{\mathsf P(D_{t,j})}\int_{D_{t,j}} Z\,\mathrm{d}\mathsf P:
D_{t,j}\subseteq E
\right\}, \quad E\in\cE_t^0.$$
The value on the empty set is understood according to the standing convention. This operator evaluates an event $E$ by the worst conditional mean over the information atoms contained in $E$. Hence, the condition
$\sA^{\textrm{atom}}_t(Z|E)\ge u$
means that every information atom included in $E$ has conditional mean at least $u$.

The operator $\sA^{\textrm{atom}}_t$ is nonincreasing with respect to sets, since enlarging $E$ adds further atoms to the minimum. It also has the reduction property. Indeed, if $B,E\in\cE_t^0$ and $E\not\subseteq B$, then, because both $B$ and $E$ are unions of atoms, there exists an atom $D_{t,j}\subseteq E$ such that $D_{t,j}\cap B=\emptyset$. On this atom $Z\mI{B}=0$, and therefore the corresponding conditional mean is zero. Consequently, $\sA^{\textrm{atom}}_t(Z\mI{B}|E)=0.$ Thus, $\sA^{\textrm{atom}}_t$ satisfies the structural assumptions of the localization theorem.

In this case the generalized level measure has a transparent form. Put
$$m_{t,j}(Z):=\frac{1}{\mathsf P(D_{t,j})}\int_{D_{t,j}} Z\,\mathrm{d}\mathsf P. $$
Then
$\sA^{\textrm{atom}}_t(Z|E)\ge u$
holds precisely when all atoms contained in $E$ satisfy $m_{t,j}(Z)\ge u$. Hence
\begin{equation}\label{eq:atom-survival}
\sup\{\mathsf P(E) :  \sA^{\textrm{atom}}_t(Z|E)\ge u, E\in\cE_t\}
= \mathsf P\left(
\bigcup_{\{j:m_{t,j}(Z)\ge u\}}D_{t,j}
\right).
\end{equation}
Equivalently, this is the probability of a level event for the conditional mean of $Z$ with respect to the finite information structure generated by $\mathcal P_t$. If $Z=Y_t$ is measurable with respect to this finite information structure, then the atomwise conditional mean coincides with $Y_t$ on each atom, and \eqref{eq:atom-survival} reduces to the ordinary level-set probability of $Y_t$. If $Z=Y_s$ for $s>t$, then the same expression gives the level-set probability of the forecast of the future value $Y_s$ based on the finite information structure $\mathcal P_t$.

This example shows that the abstract conditional aggregation rule need not be interpreted only as an essential infimum. It may also represent a robust evaluation of conditional averages over information states. The resulting generalized level measure then describes the largest observable event on which the process satisfies a prescribed lower bound in this information-dependent sense.

\section*{Conclusion}

We have studied a localization principle for generalized level measures generated by conditional aggregation operators. The motivation comes from knowledge-based, fuzzy and non-additive systems in which a structured score is evaluated through admissible knowledge contexts and local aggregation rules. The main result shows that the equality between the generalized level measure of the masked signal $f\mI{B}$ and the localized generalized level measure of $f$ is governed by two independent structural requirements. The first one is monotonicity with respect to sets, which controls the passage from a set to its localized part. The second one is the reduction property, which ensures that after multiplication by $\mI{B}$ no nonempty set outside $B$ can still produce a positive level value.

Several mechanisms leading to the reduction property were identified. The pointwise annihilator is the simplest sufficient condition, but it is not necessary in general. In particular, block-generated collections show that localization may hold even when zeros do not annihilate the conditional aggregation rule at the point level. In such cases the reduction is produced by the structure of the collection together with an outer aggregation mechanism which propagates zero values from whole blocks. Point-separating collections represent the opposite situation: there the reduction property is equivalent to the existence of a pointwise annihilator.

The parameterized version of the theory shows that the external parameter and the level threshold must be kept separate. A parameter $t$ may determine the admissible collection, the monotone measure, and the conditional aggregation rule, whereas the level $u$ determines the threshold at which the generalized level measure is evaluated. This leads naturally to a two-variable generalized level surface
$$
\Llambda_{\mathfrak S}(f;t,u)
=
\sup\{\mu_t(E):\sA_t(f|E)\ge u,\ E\in\cE_t\}.
$$
The diagonal choice $u=t$ may be useful in applications, but it is only a specialization of the two-variable formulation. This distinction is essential, since the proof of monotonicity requires testing the localization identity at levels determined by the values of the conditional aggregation rules themselves, not only at prescribed diagonal levels. When $T\subseteq(0,\infty)$, the diagonal
$$
S_f(t):=\Llambda_{\mathfrak S}(f;t,t)
=
\sup\{\mu_t(E):\sA_t(f|E)\ge t,\ E\in\cE_t\}
$$
may be interpreted as a level-dependent evaluation curve. Under the structural assumptions of Theorem~\ref{thm:parametric}, one obtains
$$
S_{f\mI{B}}(t)=S_f^B(t),
$$
where
$$
S_f^B(t):=\Llambda_{\mathfrak S}(f;t,t;B)
=
\sup\{\mu_t(E\cap B):\sA_t(f|E)\ge t,\ E\in\cE_t\},
$$
for all positive levels for which $B$ belongs to the corresponding collection. This is the level-dependent analogue of the classical identity
$$
\mu(\{f\mI{B}\ge t\})=\mu(\{f\ge t\}\cap B).
$$

The knowledge-based reading shows that the same theorem supplies a consistency principle for localized evaluation under structured uncertainty. When hypotheses, configurations or alternatives are represented by structured support signals, the theorem identifies when comparing masked signals is equivalent to comparing localized admissible evidence contexts. In particular, localization-constrained selection problems can be formulated either through localized contexts or through masked input signals, and the two formulations coincide precisely under the structural assumptions isolated in the paper.

The diagonal curves also provide a natural input for Choquet--Stieltjes-type integral functionals. If $\varphi\colon[0,\infty)\to[0,\infty)$ is nondecreasing and $m_\varphi$ is the associated Stieltjes measure, one may formally consider
$$\mathbf{ChS}_{\mathfrak S,\varphi}(f) := \int_{(0,\infty)} S_f(t)\,\md m_\varphi(t),$$
whenever the integral is well-defined. The localization identity then yields the representation
$$\mathbf{ChS}_{\mathfrak S,\varphi}(f\mI{B}) =
\int_{(0,\infty)} S_f^B(t)\,\md m_\varphi(t),$$
provided that $B$ is available at the relevant levels and the integrability assumptions are satisfied. Thus, the localization theorem supplies a structural basis not only for generalized level analysis, but also for level-dependent non-additive integral scores built from localized evaluation information.

The framework also admits a stochastic interpretation. For a filtered probability space and an adapted nonnegative process, sets in the corresponding collections may be interpreted as events observable at a given time, while the conditional aggregation rule may encode an essential lower bound, a robust conditional mean over information blocks, or another local performance criterion. These observations suggest several directions for further work. One may study additional classes of conditional aggregation rules satisfying the two structural conditions, develop almost-sure versions of the localization theorem for measure-theoretic settings, or impose regularity assumptions on the parameterized system in order to obtain measurable and integrable generalized level curves. Such developments would connect the present order-theoretic localization principle with knowledge-based evaluation, fuzzy and non-additive systems, non-additive integration, stochastic-process models, and applications in structured time-dependent data.

\section*{Acknowledgments}
This work was supported by the Slovak Research and Development Agency under contract No.~APVV-21-0468. The second author also acknowledges the support of the internal scientific grant vvgs-2026-3897.

\section*{Disclosure statement}
No potential conflict of interest was reported by the authors.

\section*{Data availability statement}
No datasets were generated or analyzed during the current study.

\end{document}